\pdfoutput=1

\documentclass[11pt]{article}

\usepackage{acl}

\usepackage{times}
\usepackage{latexsym}

\usepackage[T1]{fontenc}

\usepackage[utf8]{inputenc}

\usepackage{microtype}

\usepackage{inconsolata}

\usepackage{graphicx}

\usepackage{times}  
\usepackage{helvet}  
\usepackage{courier}  
\usepackage{url}  
\usepackage{natbib}  
\usepackage{caption} 
\usepackage{algorithm}
\usepackage{algorithmic}
\usepackage{amsmath}
\usepackage{amssymb}
\usepackage{booktabs}
\usepackage{mathtools}
\usepackage{amsthm}
\usepackage{xspace} 
\usepackage{graphicx}
\usepackage{subcaption}
\usepackage{wrapfig}
\usepackage{multirow}
\usepackage{makecell}
\usepackage{nameref}

\usepackage{newfloat}
\usepackage{listings}
\DeclareCaptionStyle{ruled}{labelfont=normalfont,labelsep=colon,strut=off} 
\floatstyle{ruled}
\newfloat{listing}{tb}{lst}{}
\floatname{listing}{Listing}

\theoremstyle{plain}
\newtheorem{theorem}{Theorem}[section]

\newtheorem{lemma}[theorem]{Lemma}

\theoremstyle{definition}

\theoremstyle{remark}

\newcommand{\name}{\textsc{JPO}\xspace}
\newcommand{\dpo}{\textsc{DPO}\xspace}
\usepackage{amsmath, amsthm}
\newcommand\numberthis{\addtocounter{equation}{1}\tag{\theequation}}

\usepackage{xcolor}
\usepackage{titlesec}

\title{Comparing Bad Apples to Good Oranges:\\ Aligning Large Language Models via Joint Preference Optimization}

\author{Hritik Bansal$^{\heartsuit*}$ \ \ \ \ \
Ashima Suvarna$^{\heartsuit*}$ \ \ \ \ \
Gantavya Bhatt$^{\diamondsuit*}$
\\
{\bf Nanyun Peng$^{\heartsuit}$\ \ \ \ \
Kai-Wei Chang$^{\heartsuit}$ \ \ \ \ \
Aditya Grover$^{\heartsuit}$ \ \ \ \ \
} \\
 $^{\heartsuit}$University of California, Los Angeles
  $^{\diamondsuit}$University of Washington, Seattle \\
\texttt{\{hbansal,asuvarna31\}@cs.ucla.edu,\{gbhatt2\}@uw.edu} \\
  }

\begin{document}
\maketitle

\begin{abstract}
A common technique for aligning large language models (LLMs) relies on acquiring human preferences by comparing multiple generations conditioned on a fixed context. This method, however, relies solely on pairwise comparisons, where the generations are evaluated within an identical context.  While effective to such conditional preferences often fail to encompass the nuanced and multidimensional nature of human preferences. In this work, we revisit the traditional paradigm of preference acquisition and propose a new axis based on eliciting preferences jointly over the instruction-response pairs. Unlike prior preference optimizations, which are designed for conditional ranking protocols (e.g., DPO), we propose Joint Preference Optimization (\name{}), a new preference optimization objective that upweights the joint probability of the chosen instruction-response pair over the rejected instruction-response pair. Interestingly, LLMs trained with joint instruction-response preference data using \name{} outperform LLM trained with DPO by $5.2\%$ and $3.3\%$ win-rate for summarization and open-ended dialogue datasets, respectively. Our findings reveal that joint preferences over instruction and response pairs can significantly enhance the alignment of LLMs by tapping into a broader spectrum of human preference elicitation. The data and code is available at
\url{https://github.com/Hritikbansal/dove}.


\end{abstract}

\section{Introduction}
\label{sec:introduction}

Recently, alignment \citep{stiennon2020learning,ouyang2022training} has emerged as a crucial step in enhancing the performance of large language models (LLMs) \citep{TheC3, openai2023gpt4, team2023gemini, claude, brown2020language, touvron2023llama, jiang2023mistral} in diverse real-world applications \citep{alpaca_eval,zheng2023lmsys,wu2023bloomberggpt,clusmann2023future,lambert2024rewardbench}. In particular, aligned LLMs generate responses that maximize human utility along various dimensions such as helpfulness, coherence, and harmlessness \citep{askell2021general,ouyang2022training}. Here, the notion of human utility is subjective \cite{kirk2024prism, gabriel2020artificial}, and mainly hinges on \textit{how} preferences are acquired from annotators \cite{otto2022context}.
Among the various preference acquisition protocols \citep{lightman2023let,wu2023fine,scheurer2023training,bansal2023peering}, the ranking-based approach is the most widely used paradigm to align LLMs \citep{stiennon2020learning,ouyang2022training,bai2022training,tunstall2023zephyr,OpenHermes25}. Specifically, in ranking approach the annotator has to compare a pair of responses \textit{conditioned} on a fixed context. For instance, humans can select a `preferred' response by comparing a pair of responses for the instruction `Create a list of four fruits other than Apple' (Figure \ref{fig:main} (\textit{left})).

\begin{figure*}[t]
\begin{center}
\includegraphics[scale=0.6]{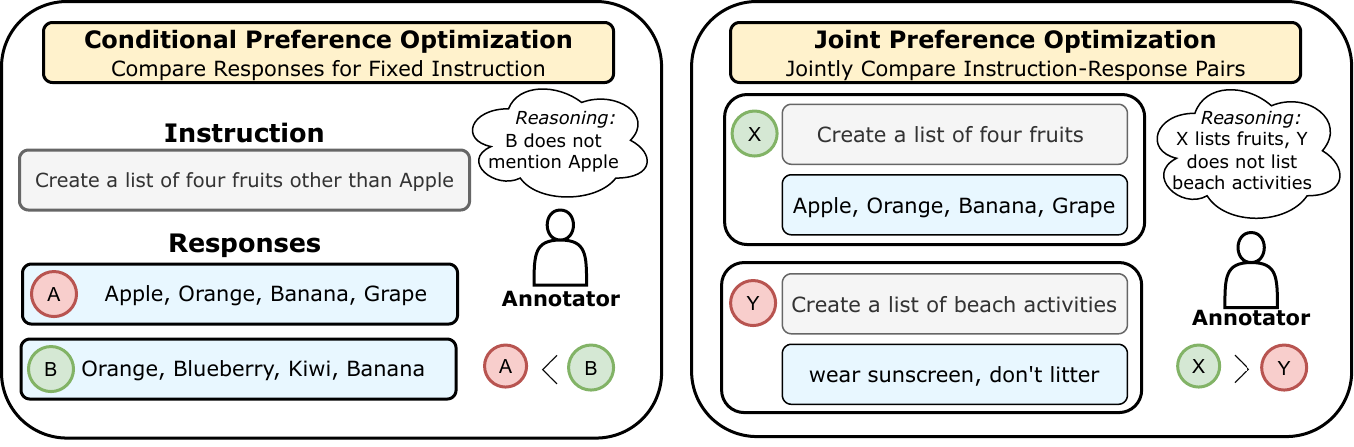}
\end{center}
\caption{\small{Overview of the Joint Preference Optimization. (\textit{Left}) We show that the conditional preference acquisition method would require the annotators to compare two responses for an identical instruction. (\textit{Right}) We show that the annotators can also assign rankings jointly over instruction-response pairs. Specifically, the annotator prefers a helpful response (e.g., Apple ... Grape) over a response that ignores the context of the instruction (e.g., wear sunscreen ... litter). Our framework thus elicits preferences that are obfuscated in the prior approach.}}
\label{fig:main}
\end{figure*}

Although traditional conditional rankings provide rich preferences for alignment, they fail to holistically capture the multi-faceted nature of  human decision-making and preferences \citep{zhi2024beyond, gigerenzer2008heuristics}. Besides ranking preferences conditioned on a fixed context, humans can also express preferences in non-identical contexts. For example, while browsing reviews for products on an e-commerce website, humans are likely to prefer an accurate and detail-oriented review for a camera over an incoherent, vague movie review even though the products (camera and movie) are qualitatively different. 


In this work, we revisit the traditional paradigm of conditional preference acquisition by developing a framework to acquire preferences jointly over instruction-response pairs. Starting from an instruction-response data consisting of response $R_i$ for instruction $I_i$ (say $i \in \{1,2\}$), we acquire ranking-based preferences over the instruction-response pairs $(I_1, R_1)$ and $(I_2, R_2)$. As shown in Figure \ref{fig:main} (\textit{right}), we aim to understand whether the response in the pair $X$ is perceived better than the response in the pair $Y$. We hypothesize that by capturing preferences in non-identical contexts our protocol can elicit human behaviors that are obfuscated in prior protocols. First, wee show that humans can provide decisive preferences in joint preferences protocol (\ref{sec:interplay_analysis}). Then, we analyze how joint preferences differ from conditional preferences on the same dataset(\S \ref{sec:interplay_analysis}).


\begin{table}[ht]
\centering
\large
\resizebox{\columnwidth}{!}{
\begin{tabular}{cccc}
\toprule
\textbf{\parbox[c]{3cm}{\centering Preference Acquisition}} & \textbf{Algorithm} & \textbf{\parbox[c]{3cm}{\centering Alignment Objective}} & \textbf{\parbox[c]{3cm}{\centering Different Instructions}} \\
\midrule
Score           & \citet{ethayarajh2024kto}   & Conditional & No  \\ \hline 
\multirow{5}{*}{ \parbox[c]{3cm}{\centering Comparison \\[0.5ex] (DPO Variants)} }   & \citet{dpo} & Conditional & No  \\
                     & \citet{park2024disentangling} & Conditional & No  \\
                     & \citet{liu2024lipo}  & Conditional & No  \\
                     & \citet{meng2024simposimplepreferenceoptimization} & Conditional & No  \\
                     & \citet{hong2024orpomonolithicpreferenceoptimization}  & Conditional & No  \\
\midrule
Pairwise  & \name (ours) & Joint       & Yes \\
\hline
\end{tabular}
}
\caption{\name differs from prior works along three key aspects: preference acquisition (scoring or comparison), objective (conditional or joint distribution), and their ability to handle non-identical instruction-responses.}
\label{tab:related_works}
\end{table}

Prior works like DPO and its variants \cite{rafailov2023direct,yin2024relative,liu2024lipo, meng2024simposimplepreferenceoptimization, hong2024orpomonolithicpreferenceoptimization, azar2023general} rely on conditional rankings, and thus do not have access to the joint distribution of human preferences in the ranking protocol (Table~\ref{tab:related_works}). While a rating protocol \cite{ethayarajh2024kto} allows for a comparison between responses from non-identical instructions, it can be inconsistent with rankings \cite{bansal2023peering} and ignores the possibility of preferences over a pair of chosen or rejected responses. \footnote{For instance, a pair of responses that achieves a score of 0, under the rating protocol, will result in an indecisive preference.} Thus, we propose \textbf{Joint Preference Optimization} (\name), a framework for aligning LLMs with our proposed joint preference elicitation scheme. Specifically, it upweights the joint probability of the chosen instruction-response pair over the rejected instruction-response pair. Furthermore, \name subsumes prior preference optimizations as conditional rankings are a special case of joint preferences (e.g., when the instructions are identical). 




We conduct experiments to explore new reasoning paths enabled by joint preference elicitation and alignment of LLMs with the \name objective. By analyzing feedback from conditional rankings and joint preferences protocols, along with explanations from human annotators, we uncover the complexities of the preference acquisition process (\S \ref{sec:interplay}). Using our \name algorithm, we align a Mistral-7B LLM with the collected preferences, achieving a $30\%$ and $18\%$ higher win rate against gold responses on unseen instructions from summarization and dialogue datasets, respectively. \name leverages diverse preferences effectively, outperforming \dpo by $5.2\%$ and $3.3\%$ win-rate points on the summarization and open-ended dialogues, respectively. In addition, \name outperforms KTO by $3.5\%$ on the open-ended dialogues dataset. It also surpasses both DPO and KTO in the AlpacaEval2 benchmark (\S \ref{sec:experiments}). This indicates that by utilizing the diverse preference signals present in the existing data, we can align an LLM robustly without acquiring additional instruction-response data. 

\section{Related Work}
\label{sec:detailed_relatedwork}


\paragraph{Alignment using Reinforcement Learning.} Aligning LLMs with human preferences using reinforcement learning is widely adopted to ensure LLMs follow user intents without being harmful \cite{ouyang2022training}. This alignment is usually done by first optimizing for a reward model on preference data \cite{bradley1952rank, likert1932technique, bansal2023peering}, followed by aligning the LLMs distribution that maximizes the learned reward model using Reinforcement Learning (RLHF) \cite{schulman2017proximal, ouyang2022training}, with optional Divergence penalty \cite{wang2023beyond} to avoid deviating from the reference policy. Additionally, \cite{dubois2023alpacafarm, lu2024llmscore, zheng2023judging}  observe that preferences from LLMs can also be used for alignments motivating Reinforcement Learning through AI feedback (RLAIF).  Contrary to prior work that collect preferences as conditional rankings, we emphasize that preference acquisition is a complex phenomenon and elicit joint preferences over instruction-response data. 


\paragraph{Reward Free Policy Alignment.} \citet{dpo} introduced Direct Preference Optimization (DPO) that optimizes directly within the model parameter space, hence eliminating the reward modeling step. \cite{liu2024lipo} extends this framework where instead of two responses, alignment is done over the list of responses while \cite{liu2023statistical} improves DPO using statistical rejection sampling. \cite{amini2024direct} provides an offset in the DPO objective to increase the margins and \cite{pal2024smaug} suggests adding an explicit penalty term to avoid a reduction in the likelihood of preferred pairs over the DPO training. Recent variants of DPO such as SimPO \citep{meng2024simposimplepreferenceoptimization} alleviates the need of reference policy in the objective. Contrary to our work where we compare the joint distributions, \cite{yin2024relative} proposes RPO that compares the conditional likelihood of a winning response with the losing response of another prompt. Beyond DPO, \cite{ethayarajh2024kto} proposed a human-aware loss function-based framework using prospect theory named KTO, and \cite{azar2023general} proposes IPO that uses human preferences expressed as pairwise preferences. Lastly, \cite{zhao2022calibrating} uses sequence likelihood calibration to align the model from human preference. Despite of a vast body of work arising from DPO, none of the existing methods can operate and contrast over the joint distribution of instruction-response pairs like the proposed \name algorithm. 

\section{Background}
\label{sec:background}

In this work, we focus on aligning language models to generate outputs preferred by humans in dimensions like helpfulness and coherence. Aligning a pretrained base model involves four steps: (a) collecting instruction-response data, (b) supervised fine-tuning (SFT), (c) acquiring preference data, and (d) deploying an alignment algorithm. The instruction-response data can be either hand-crafted by humans \citep{DatabricksBlog2023DollyV2,wang2022super} or generated by AI \citep{alpaca,tunstall2023zephyr}. Subsequently, the base model undergoes supervised fine-tuning (SFT) on the instruction-response pairs \citep{zheng2023judging, wang2023selfinstruct, wang2022super, peng2023instruction, xu2023wizardlm, yin2023dynosaur, wang2023far,yu2023metamath,toshniwal2024openmathinstruct}. 

Following SFT, feedback data is gathered (e.g., rankings) to train the SFT model via an alignment algorithm. This often involves training a reward model on preference data \citep{bradley1952rank, bansal2023peering} and aligning the model using Reinforcement Learning (RLHF) \citep{schulman2017proximal, ouyang2022training}. To address challenges in human feedback collection \citep{dubois2023alpacafarm, zheng2023judging}, LLMs can provide feedback, enabling Reinforcement Learning through AI Feedback (RLAIF). Alternatively, \cite{dpo} introduced Direct Preference Optimization (DPO) that mitigates the instability due to PPO for reward maximization by optimizing directly within the model parameter space, hence by-passing the reward modeling step.

\section{Joint Preference Optimization (\name)}
\label{sec:approach}


 A common technique for feedback data acquisition requires the annotators to assign a preferred and non-preferred label to a pair of responses for an instruction \citep{stiennon2020learning, rafailov2023direct, ouyang2022training, ethayarajh2024kto}. However, this paradigm does not capture the complex and multidimensional aspects of human preferences \citep{kendall1940method,thurstone2017law, zhi2024beyond}. Specifically, the heuristics for making preference decisions depend upon the context in which the comparison is made \cite{otto2022context}. While the traditional ranking protocol compares the two responses under a fixed context, humans can perform pairwise comparisons jointly over instruction-response pairs. For example, consider two summaries, A and B, for articles X and Y, respectively; then, a human can reason and choose the response that better summarizes its corresponding article. Hence, it is critical to align language models with diverse feedback signals to accurately model human behavior and decision making. 

In our setup, the annotator has to decide a \textit{chosen} and \textit{rejected} instruction-response pair $(I_a, R_a, I_b, R_b)$ where $R_a$ and $R_b$ are responses to the instructions $I_a$ and $I_b$, respectively, and $(I_a, R_a), (I_b, R_b) \in \mathcal{D}$. We note that our joint preference setup is equivalent to the original ranking protocol when $I_a = I_b$. As before, the preference reasoning from the annotator will be based on subjective dimensions like helpfulness, coherence, and harmlessness. Formally, the annotator assigns a joint ranking feedback $h(I_a, R_a, I_b, R_b) \in \{(I_a, R_m), (I_b,R_b), \text{Equal}\}$ where `Equal' indicates that both the instruction-response pairs are perceived equally good or bad. Finally, the joint preference optimization creates a pairwise feedback data $\mathcal{D}_H = \{(I_a, R_a, I_b, R_b, h(I_a, R_a, I_b, R_b))\}$. 

Our formulation suggests that we can obtain large-scale and diverse preference data (covering all possible combinations of $(I_a, R_a)$ and $(I_b, R_b)$) without the need for gathering additional instruction and response data, which is typically more difficult and costly to acquire. In addition, joint preference acquisition does not necessitate the presence of multiple responses for a given instruction that can be hard to collect for low-resource languages (e.g., Kalamang \footnote{\url{https://endangeredlanguages.com/lang/1891?hl=en}}). Specifically, one can collect an instruction-response data $\mathcal{D'} = \{(I_a, R_a)\}_{a=1}^{a=n}$, and acquire preferences on various combinations of instruction-response pairs. Finally, we assess the interplay between the joint feedback dataset $\mathcal{D}_H$ with the conditional feedback dataset $\mathcal{D}_C$ along with qualitative examples in \S \ref{sec:interplay}. 


We propose \name, a preference optimization objective that learns to align the language models with preferences acquired jointly over the instruction-response pairs. We assume a joint preference dataset $\mathcal{D}_X = \{(I_i^w, R_i^w, I_j^\ell, R_j^\ell)\}$, that can be constructed from $\mathcal{D}_H$, where $(I_i^w, R_i^w)$ and $(I_j^\ell, R_j^\ell)$ are the chosen and rejected instruction-response pairs, respectively. Similar to \dpo, we start with a reference model $p_{\text{ref}}$ which is usually the supervised finetuned language model $p_{\text{sft}}$. Specifically, the \name objective aims to learn an aligned model $p_\theta$ by upweighting the joint probability of preferred responses $p(R_i^w , I_i^w)$ over non-preferred responses $p(R_j^\ell , I_j^\ell)$. Formally, the optimization objective for \name, $\mathcal{L}(\theta; \mathcal{D}_X, \beta, p_{\text{ref}}) $ minimizes the expectation over $(I_j^w, R_j^w, I_j^\ell, R_j^{\ell}) \sim \mathcal{D}_X$ :

{\footnotesize
\begin{align}
     \mathbb{E}\left[ \log\left(\sigma\left(\beta \log\frac{p_{\theta}(R_i^w, I_i^w)}{p_{\text{ref}}(R_i^w, I_i^w)} - \beta \log\frac{p_{\theta}(R_j^\ell, I_j^\ell)}{p_{\text{ref}}(R_j^\ell, I_j^\ell)}\right)\right)\right] \label{eq:2} \numberthis
\end{align}
}%

where $\sigma$ denotes the sigmoid function and $\beta$ is a hyperparameter. Further, we show that Eq. \ref{eq:2} reduces to the \dpo formulation (Appendix Eq. \ref{eq:1}) when the instructions $I_i = I_j$ in Appendix \S \ref{sec:equivalence}. We can also see that the \name objective aims to learn an aligned model $p_\theta$ by upweighting the conditional probability of preferred responses $p(R_i^w | I_i^w)$ over non-preferred responses $p(R_j^\ell | I_j^\ell)$, along with a correction factor based on the prior probability of the instructions under the language model $p_\theta(I_i^w)$ and $p_\theta(I_j^\ell)$. In \S \ref{sec:experiments}, we utilize \name to align language models to generate human-preferred summaries and answer open-ended instructions.

\subsection{Comparison of Joint Preferences with Prior Preference Protocols}
\label{sec:app_prior_work}

\name improves over prior work by acquiring ranking-based preferences over non-identical instructions that has remained unexplored in prior work (please refer to table~\ref{tab:related_works}). Diverse human reasoning cannot be captured in the traditional conditional framework it fails to capture human preferences over varied contexts. Context influences decision-making and subjective valuation when capturing human preferences \citep{otto2022context}.
Prior work \cite{yin2024relative,liu2024lipo, meng2024simposimplepreferenceoptimization, hong2024orpomonolithicpreferenceoptimization} collect conditional preferences in a pairwise manner and are variants of DPO \cite{rafailov2023direct}. Thus, in our experiments we compare \name to DPO directly. Furthermore, we implement KTO \cite{ethayarajh2024kto} as a baseline since KTO removes the requirements of preference data that should be paired in preference optimization and implicitly compares responses from different instructions. We find that \name outperforms both DPO and KTO.
\section{Interplay between Feedback Protocols}
\label{sec:interplay}

\subsection{Instruction-Response Acquisition}
\label{sec:ir_data}


In this work, we consider two kinds of instruction-response data. First, we consider a filtered version of the TL;DR \textit{summarization} dataset \citep{volske2017tl} from \cite{stiennon2020learning} consisting of Reddit posts, their summarizes, and human preferences over a pair of summaries for a given post. Throughout the dataset, the task is of summarization that is close-ended and well-defined for language models. Second, we consider the single-turn dialogues from the helpful-base subset of the Anthropic-HH dataset \citep{bai2022constitutional}. Specifically, this dataset consists of \textit{open-ended} instructions with a collection of responses.  

Both these datasets have a train and test split where each instance consists of an instruction and a pair of responses $\mathcal{D} = \{(I_i, R_i^1, R_i^2)\}_{i=1}^n$ where $n$ is the dataset size. In this work, we collect AI and human feedback on the instruction-response data from their train split and filter instances with duplicate instructions. We can directly compare the two responses for the fixed instruction and construct a ranking feedback dataset $\mathcal{D}_C = \{(I_i, R_i^1, R_i^2, c(I_i, R_i^1, R_i^2))\}$. To acquire preferences jointly over the instruction-response pairs, we randomly select one of the responses from every instance of $\mathcal{D}$ to construct $\mathcal{D}_{S} = \{(I_i, R_i)\}$ where $R_i \in \{R_i^1, R_i^2\}$. Subsequently, we create the joint instruction-response pairs by matching every instance $(I_i, R_i) \in \mathcal{D}_S$ with another instance $(I_j, R_j) \in \mathcal{D}_S$ to get $\mathcal{D}_H = \{(I_i, R_i, I_j, R_j, h(I_i, R_i, I_j, R_j))\}$ of the same size as $\mathcal{D}_S$ and $\mathcal{D}_C$. In \S \ref{sec:experiments}, we will utilize $\mathcal{D}_S$ to SFT the base model, and $\mathcal{D}_C$ and $\mathcal{D}_H$ as preference datasets for LLM alignment. We provide the dataset statistics in Appendix \S \ref{sec:data_stats}.



\subsection{Feedback from AI and Humans}
\label{sec:feedback_ai_human}
\begin{table}[h]
\centering
\resizebox{\linewidth}{!}{%
\begin{tabular}{lccc}
\hline
\textbf{Dataset}            & \textbf{Ranking}                & \textbf{H-H} & \textbf{H-AI} \\
\hline
TL;DR   &     \multirow{2}{*}{\makecell{Conditional}}              & 69\%        & 63\%               \\
Anth.-Helpful &   & 70.1\%      & 72\%             \\\hline
TL;DR     &  \multirow{2}{*}{\makecell{Joint\\(Non-Identical)}}       & 62\%        & 60\%               \\
Anth.-Helpful & & 68.8\%      & 71\%    \\\hline          
Average & & 67.5\%  & 66.5\%    \\\hline          
\end{tabular}
}
\caption{Agreement analysis between within human annotators and gold human feedback and AI (ChatGPT) feedback. We perform the agreement calculations for the two ranking protocols: (a) conditional rankings, and (b) joint preferences where instructions are non-identical. In addition, we assess the agreement rates over the two datasets: (a) TL;DR and (b) Anthropic-helpful dataset.}
\label{tab:agreement}
\end{table}

\paragraph{Feedback from AI.} 
We collect feedback over a pair of responses for a fixed instruction, and joint instruction-response pairs without identical instructions from GPT-3.5-Turbo-0125 (ChatGPT). We choose ChatGPT due to its affordability  (e.g., output tokens from ChatGPT are $50 \times$ cheaper than GPT-4).
To mitigate any ordering bias, we run two queries for all comparisons. When the ChatGPT preferences flip by flipping the order of the two responses, then we consider it a tie, similar to \citep{bansal2023peering,bitton2023visitbench}. Specifically, we instruct the AI to to choose the response that is more accurate, coherent, and harmless.

To collect conditional preferences over a pair of responses for a fixed instruction, we prompt ChatGPT to choose a response. To collect AI preferences jointly over the instruction-response pairs, we prompt ChatGPT to decide the response that better answers its corresponding instruction. We collected approximately $50K$ comparisons across both feedback acquisition protocols for the summarization and Anthropic-Helpful dataset, at a cost of $\$100$. We provide the AI prompts in Appendix \S \ref{sec:gpt_prompts}.

 

\paragraph{Feedback from Humans.} In this work, we also collect human preferences for $2000$ comparisons over TL;DR and Anthropic-Helpful dataset. Specifically, we ask two annotators to assign a chosen response or chosen instruction-response pair based along the same dimensions as ChatGPT guidelines. Annotators can also choose `equal' if they fail to make a identify a decisive preference. The human annotations were collected from Amazon Mechanical Turk (AMT). We recruited the participants that passed a preliminary qualification exam. 
In total, we spent $\$720$ on human feedback acquisition. We provide the screenshot of the annotation UI in Appendix \S \ref{sec:human_annotation}.



\subsection{Agreement Analysis} 
\label{sec:agreement_analysis}

\begin{figure}[t]
    \centering
    \includegraphics[width=\linewidth]{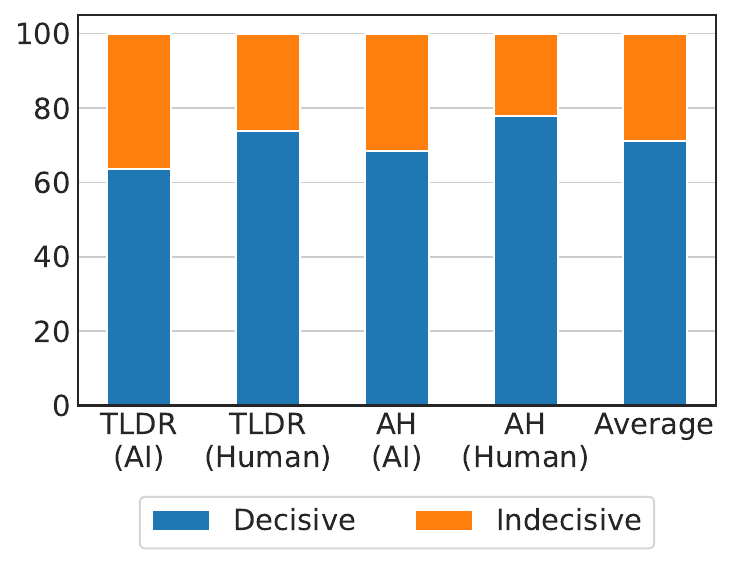}
    \caption{Results for the preferences acquired jointly over the instruction-response pairs where both the responses were either chosen or rejected under the conditional rankings protocol. Here, \textit{decisive} implies that the annotators could assign a preference to one instruction-response pair over the other. Here, AH means Anthropic-Helpful. }
    \label{fig:cc_rr}
\end{figure}

\begin{figure}[t]
    \centering
    \includegraphics[width=\linewidth]{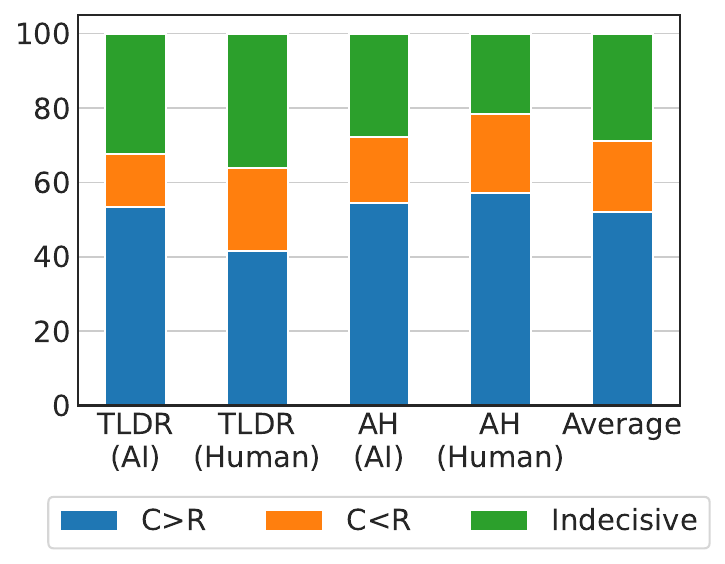}
    \caption{Results for the preferences acquired jointly over the instruction-response pairs where one of the instruction-response pair was chosen (C) and the other pair was rejected (R) under the conditional rankings. Here, $C < R$ implies that the instruction-response pair that was rejected under conditional rankings is actually preferred over an instruction-response pair that was rejected under the conditional rankings. Here, AH means Anthropic-Helpful. }
    \label{fig:cr_rc}
\end{figure}

We present the annotator agreement scores in Table \ref{tab:agreement}. We find that the average agreement is $67.5\%$ and $66.5\%$ between the human-human and human-AI annotators, respectively. 
Specifically, we find that in the conditional setup (identical instructions), the average human-human agreement is 69.5\% for the TLDR and Anthropic-Helpfulness datasets. Similarly, in the joint setup (non-identical instruction-response pairs), the average inter-rater agreement is 68\% on the same datasets.These agreement scores are comparable to those reported in prior studies \citep{alpaca_eval,bansal2023peering}, demonstrating the robustness of our evaluation. Interestingly, we find that agreement scores vary depending on the underlying distribution of instruction-response pairs and the choice of ranking protocol. Overall, our results highlight that humans and AI can provide rich feedback in both conditional and joint setup with acceptable agreement.

\begin{table*}[h]
\caption{Response B and D are rejected under the conditional rankings. However, in the the joint setup humans consider that the response B answers Instruction 1 better than response D answers Instruction 2. Thus, joint setup human give a decisive feedback between two responses that were rejected under the conditional rankings. }
\centering
\resizebox{\linewidth}{!}{%
\begin{tabular}{p{0.2\linewidth}| p{0.8\linewidth}}
\toprule
\small{\textit{Instruction 1}}& \small{My dog has been scratching a lot, but he doesn't have any fleas. What else could be causing this?}                                                                                                                                                                                           \\
\small{\textit{Response A}}& \small{Fleas are the most likely cause of a dog's scratching, but there are many other possible reasons. It's worth taking him to the vet to check things out.}                                                                                                                                                \\
\small{\textit{Response B}}& \small{Oh no! Maybe his skin is dry. I bet if you gave him a good bath, that would help!}                                                                                                                                                                                                                      \\
\small{\textit{Conditional Preference 1}} & \small{\textbf{Human Preference: A>B.}  Human Explanation: Response A answered the query by giving possible causes of the problem along with a suggestion of just going to the vet.}                                                                            \\\hline
\textit{Instruction 2}            & \small{Do you have any good chili recipes?}                                                                                                                                                                                                                                                            \\
\small{\textit{Response C}}& \small{I’m happy to give some suggestions and ideas, but I’m not really an expert on chili. What kind of chili do you want to make?}                                                                                                                                                                         \\
\small{\textit{Response D}}& \small{I’m afraid not, I don’t have any experience in cooking, and have been programmed to give answers and ask for clarification on issues in the area of cooking!}\\

\small{\textit{Conditional Preference 2}} &\small{\textbf{Human Preference: C>D.} Human Explanation: Response D claims they can't help and then later on in the same sentence claims they've been programmed to give answers on cooking issues, which doesn't make much sense. C asks a reasonable clarifying question.}\\\hline
\small{\textit{Joint preference}}&\small{\textbf{Human Preference: (Instruction 1, Response B) > (Instruction 2, Response D). }Human Explanation: Response B provides answer to the Instruction 1 while Response D does not provide an accurate answer to Instruction 2.}      \\\bottomrule                       
\end{tabular}%
}
\label{tab:rr_anthropic_helpful}
\end{table*}

\subsection{Interplay Analysis}
\label{sec:interplay_analysis}

\paragraph{Setup.}

Here, we aim to study the interaction between the conditional rankings and joint rankings over non-identical instructions. Formally, each instruction-response pair $(I_i, R_i^x) $ from the conditional pairwise feedback dataset $\mathcal{D}_C$ where $x \in \{1,2\}$ can be assigned a preference $\mathcal{P}_{C}(I_i, R_i^x)$ among \{`chosen', `reject', `equal'\}. For instance, $\mathcal{P}_C(I_i, R_i^1) = `\text{chosen}'$ and $\mathcal{P}_C(I_i, R_i^2) = `\text{reject}'$ if the response $R_i^2$ is rejected in the dataset $\mathcal{D}_C$ i.e., $c(I_i, R_i^1, R_i^2) = R_i^1$. Similarly, we can assign a preference $\mathcal{P}_{H}(I_i, R_i)$ among \{`chosen', `reject', `equal'\} to an instruction-response pair $(I_i, R_i)$ from the joint preference dataset $\mathcal{D}_{H}$. For instance, $\mathcal{P}_{H}(I_i, R_i) = `\text{chosen}'$ and $\mathcal{P}_H(I_j, R_j) = `\text{reject}'$ where $i != j$ if the instruction-response pair $(I_i, R_i)$ is chosen in the dataset $\mathcal{D}_H$ i.e., $h(I_i, R_i, I_j, R_j) = (I_i, R_i)$.

To study the interplay between the preference protocols, we assess $\mathcal{P}_{C}(I_i, R_i)$,  $\mathcal{P}_C(I_j, R_j)$, $\mathcal{P}_{H}(I_i, R_i)$ and  $\mathcal{P}_H(I_j, R_j)$ for all $(I_i, R_i, I_j, R_j) \in \mathcal{D}_H$. Here, if $\mathcal{P}_H(I_i, R_i) = `\text{chosen}'$ then $\mathcal{P}_H(I_j, R_j) = `\text{reject}'$.

\paragraph{Annotators show decisiveness in joint setup.} In Figure \ref{fig:cc_rr}, we study the joint preferences over the instruction-response pairs $(I_i, R_i, I_j, R_j)$ where the individual instruction and response data is either \textit{chosen} or \textit{rejected} in the conditional feedback protocol (e.g., $\mathcal{P}_C(I_z, R_z) = `\text{chosen}'$ for $z \in \{i,j\}$). Interestingly, we find that the annotators can assign a decisive preference (e.g., $(I_i, R_i) > (I_j, R_j)$) in $71\%$ of the joint comparisons. While we observe that the annotators assign a `tie' to $29\%$ of the comparisons. This highlights the existence of valid preference decisions that remained obfuscated in the traditional approach for ranking-based feedback acquisition.

\paragraph{Annotator preferences depend on context and comparisons.} In Figure \ref{fig:cr_rc}, we study the joint preference over the instruction-response pairs $(I_i, R_i, I_j, R_j)$  where one of them is \textit{chosen} and the other is \textit{rejected} in the conditional feedback protocol (e.g., $\mathcal{P}_C(I_i, R_i) = `\text{chosen}'$ and $\mathcal{P}_C(I_j, R_j) = `\text{reject}'$). To our surprise, we find that the annotators do not prefer the instruction-response pair that was chosen under the conditional feedback protocol in $48\%$ of the comparisons. Specifically, there are $19\%$ of the comparisons where rejected pair (R) is preferred over the chosen pair (C) and $28\%$ of the comparisons where the annotators considered the pair equally good or bad. This highlights that both human and AI annotators' perceptions of preferred and non-preferred data depends on the context of the comparisons, indicating that feedback acquisition is a multifaceted phenomenon.

\paragraph{Qualitative Case Study.} To understand the heuristics used in preference annotations, we asked human annotators to provide brief explanations for their feedback decisions in both conditional and joint preference setups. In Table \ref{tab:rr_anthropic_helpful}, we observe that humans provide reasonable explanations for rejecting responses B and D in the conditional setup. However, when these same rejected responses are presented in a joint setup, humans offer decisive feedback, basing their decisions on the accuracy of the responses—an aspect not emphasized in the explanations for the conditional preferences. We present additional qualitative examples in Appendix \S \ref{sec:append_qualitative} to showcase the multi-faceted nature of human feedback revealed through joint preferences.


\section{LLM Alignment with JPO}
\label{sec:experiments}

\begin{figure*}[t]
    \centering
    \begin{subfigure}[ht]{0.33\textwidth}
        \centering
        \includegraphics[width=\textwidth]{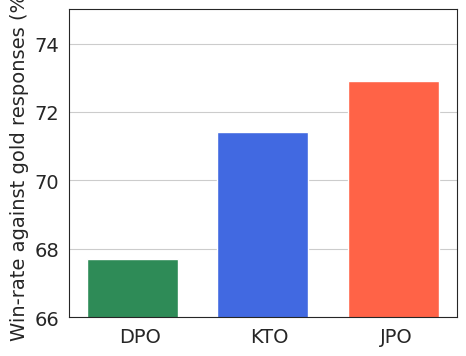}
        \caption{Performance on TL;DR}
        \label{fig:result_tldr}
    \end{subfigure}%
    \begin{subfigure}[ht]{0.33\textwidth}
        \centering
        \includegraphics[width=1.02\textwidth]{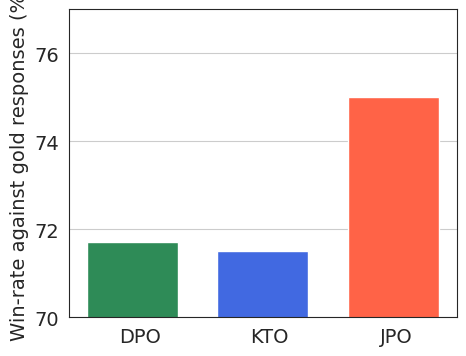}
        \caption{Performance on Anthropic-Helpful}
        \label{fig:result_anthropic}
    \end{subfigure}
    \begin{subfigure}[ht]{0.33\textwidth}
        \centering
        \includegraphics[width=\textwidth]{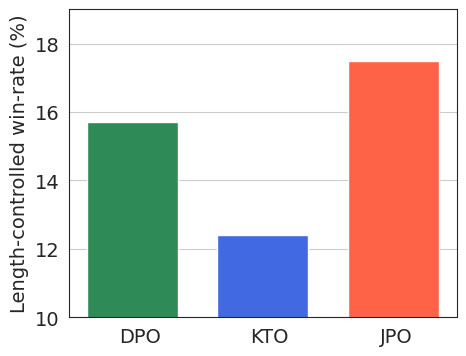}
        \caption{Performance on AlpacaEval2.0} 
        \label{fig:result_alpacaeval2}
    \end{subfigure}
    \caption{Results for aligning LLMs with \name. We utilize ChatGPT
to compare the model responses with the gold responses. In 4a and 4b we report the results averaged over three runs of the preference optimization objectives and three sampling temperatures. In 4c, we report the results for temeperature set at 0.7 for AlpacaEval2.}
    \label{fig:main_results}
\end{figure*}

In previous sections, we explore how we collect ranking-based feedback for a pair of responses for identical and non-identical instructions. Here, we study how to leverage joint and conditional feedback data to align large language models effectively with JPO \S\ref{sec:approach}. 

\subsection{Setup}
\label{sec:setup}

We align Mistral-7B ~\citep{jiang2023mistral}, a strong base LLM for its model capacity. We experiment with two datasets that exhibit diverse characteristics: (a) TL;DR dataset where the instruction is to summarize Reddit posts, and (b) open-ended dialogues from Anthropic-Helpful dataset (\S \ref{sec:ir_data}). In particular, we collect a conditional preference data $\mathcal{D}_C$ and joint preference data for non-identical instructions $\mathcal{D}_H$ of similar data sizes from ChatGPT. Then, we convert the conditional preference data into an instruction-response data for supervised finetuning $\mathcal{D}_{\text{SFT}}$.

We supervise finetune the entire base LLM model parameters with the SFT dataset to ensure that the preference data is in-policy for the alignment algorithms ~\citep{rafailov2023direct}. \name algorithm can utilize both the conditional preferences and joint preference with non-identical context. \footnote{It is because the conditional preferences can be viewed as joint preferences with identical context.} Thus, we train the base LLM with \name algorithm after merging conditional and joint preferences data $\mathcal{D}_M = \mathcal{D}_C \cup \mathcal{D}_H$. We provide more details on training setup in Appendix \S \ref{sec:alignment_training}. We also apply we apply \dpo and KTO algorithm on the SFT model to compare against \name. 



Post-alignment, we evaluate the aligned model responses against the gold responses in the dataset's test split. 
We utilize ChatGPT to compare model and gold responses to decide on the preferred response or a tie. Finally, we report the win-rate of the model responses as the evaluation metric for $500$ unseen instructions from the test sets. In particular, we report the win-rate against the gold responses for the model generated responses averaged across three sampling temperatures $\text{T} \in \{0.001, 0.5, 1.0\}$.


\subsection{Results }
\label{sec:results}




We compare the performance of the \dpo, KTO, and \name aligned models in Figure \ref{fig:result_tldr} and \ref{fig:result_anthropic}. Interestingly, we find that \name outperforms \dpo by $5.2\%$ and $3.3\%$ win-rate points on the summarization and helpfulness datasets, respectively. In addition, the performance of \name is better than \dpo across all the sampling temperatures. We observe similar trends in comparison to KTO.  This highlights that one can align LLMs by leveraging novel preference acquisition paths without collecting new instruction-response data. 

\begin{figure}[h]
  \begin{center}
    \includegraphics[scale=0.5]{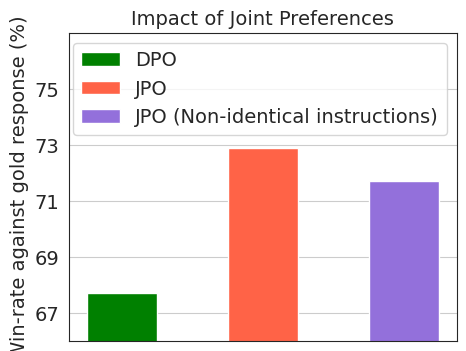}
  \end{center}
  \caption{Win-rate against the gold response in the TL;DR averaged over three sampling temperatures. We study the impact of the joint preferences over non-identical instructions using \name.}
  \label{fig:ablation_joint}
\end{figure}

\subsection{Extending the Results to AlpacaEval}
\label{sec:result_alpacaeval}

Similar to \citet{rafailov2023direct}, we show the usefulness of aligning LLMs using joint preferences via \name on close-ended (e.g., summarization) and open-ended tasks (e.g., dialogues). However, we further evaluate the effectiveness of our method on a broad set of instructions in the AlpacaEval2 leaderboard using the length-controlled win-rate metric \cite{alpaca_eval}. Additional experimental details are provided in Appendix \ref{sec:alpacaeval2}.

We present the results in Figure \ref{fig:result_alpacaeval2} where we compare \name with \dpo and KTO. We find that the \name-aligned LLM outperforms DPO-aligned LLM by $1.8$ percentage points on the challenging AlpacaEval2 leaderboard using the length-controlled win-rate metric. This indicates that the \name can utilize the joint preferences and elicit helpful and accurate responses for a broad set of instructions.

\section{Ablations}
\paragraph{Impact of Joint Preferences over Non-Identical Instructions.} Here, we aim to understand the sole impact of joint preferences acquired over non-identical instructions on the performance of the \name algorithm. To do so, we train \name algorithm with joint feedback data $\mathcal{D}_H$ only. We present the results averaged across the three sampling temperatures in Figure \ref{fig:ablation_joint}. We find that training with joint preferences over non-identical instructions achieves $71.7\%$ win-rate on the summarization  dataset. This indicates that it is possible to align LLMs with just joint preferences over instruction-response data \textit{without} any conditional preferences too. Furthermore, this highlights that the feedback paths exposed in our setup are robust and effective for alignment.

\paragraph{Impact of Dataset Size.} In the main experiments, we demonstrated that \name can learn effectively from a combination of conditional preferences (i.e., 100\% of the conditional rankings) and joint preferences over non-identical instructions (of the same size as the conditional preferences). To assess the impact of dataset size, we trained \name using a 50:50 mix of conditional and joint preferences for the TL;DR dataset, with a fixed total size as that of conditional. Our results in Figure \ref{fig:ablation_dpo_jpo_size_controlled} show that \name achieves a win rate of 71.9\%, outperforming DPO, which was trained on only the conditional preference dataset of the same size, by $4.2$ percentage points.

\begin{figure}[h]
  \begin{center}
    \includegraphics[scale=0.48]{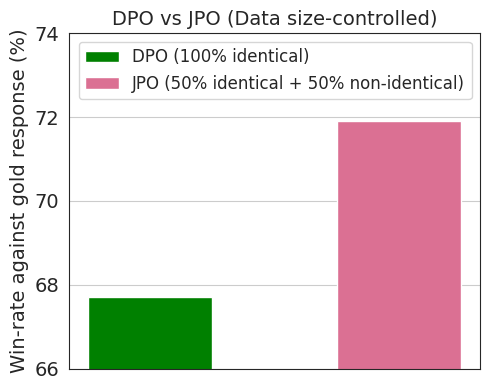}
  \end{center}
  \caption{Win-rate against the gold response in the TL;DR dataset averaged
over three sampling temperatures. We study the impact of dataset size on \name.}
  \label{fig:ablation_dpo_jpo_size_controlled}
\end{figure}

\paragraph{Data Scaling.}

We aim to understand the impact of increasing the number of preferences collected jointly over instruction-response pairs, for non-identical instructions, on the win-rate against the reference summaries in the TL;DR summarization dataset using \name algorithm. We present the results in Figure \ref{fig:tldr_scaling} for the sampling temperature of $0.001$. We find that the win-rate scales from $42.4\%$ to $71.7\%$ as the size of the dataset increases from $100$ to $9000$ comparisons. We also observe that the change in the win-rate is within $1\%$ when the dataset size increases from $4000$ to $9000$. This highlights that the performance gains are non-linear with the dataset size. In the future, it would be pertinent to explore techniques for selecting a subset of joint preference comparisons that result in maximum performance gains.

\begin{figure}[h]
    \centering
\includegraphics[scale=0.5]{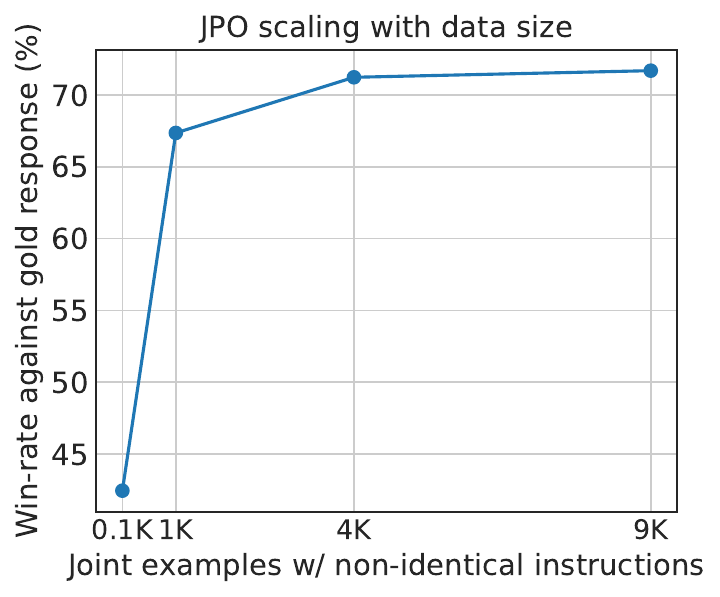}
    \caption{Results for scaling the feedback data size on TL;DR summarization dataset. We find that the win-rate improves with the increase in the dataset size using the \name preference optimization objective.}
    \label{fig:tldr_scaling}
\end{figure}
\section{Conclusion}
\label{sec:conclusion}

In this work, we propose a framework that elicits preferences jointly over instruction-response pairs. Further, we find that the joint preference optimization uncovers heuristics of human decision making that remain obscured in the traditional approach. Additionally, we propose \name, a novel preference optimization objective for aligning LLMs. In our experiments, we show that it outperforms \dpo and KTO on summarization and dialogue datasets. \name also outperforms \dpo and KTO on AlpacaEval2. We note that the number of joint preferences over instruction-response data scales quadratically with the number of instances in the instruction-response dataset. Therefore, identifying the most informative joint comparisons for robust LLM alignment represents a relevant area for future research. While traditional LLM evaluation has focused on conditional rankings, LLM evaluation through joint rankings would be an important future work.

\section{Limitations}
\label{sec:limitations}

While there are various protocols for feedback acquisition, our work is focused on acquiring rankings on a pair of responses under a fixed context or jointly over instruction-response pairs. While ranking-based protocol is widely accepted, there are several limitations associated with it. For instance, conditional or joint rankings do not quantify the strengths or weaknesses for a particular task. In addition, \cite{bansal2023peering} show that different forms of feedback data often disagree with each other. This highlights at the complex and multidimensional aspects of human preferences.

In our work, we propose the joint acquisition of feedback for pairs of instruction-response over diverse tasks (e.g., comparing a movie review with an e-commerce product review). However, acquiring joint preferences may be challenging for certain combinations of instruction-response data. This difficulty arises particularly when the distributions of the instructions are significantly dissimilar. For example, it may be challenging to compare feedback for a response to the instruction 'how to cook fried rice?' with a response to 'how to steal my neighbor's wifi?'. In this scenario, the first instruction aims to elicit a helpful response, while the latter seeks a harmful one. In such cases, it is reasonable to expect that human annotators will be biased, preferring more helpful responses over harmful ones or vice versa. Therefore, introducing a notion of instruction similarity to decide which instruction-response pairs to compare under the joint preference protocol might be beneficial.

Finally, we acquire human annotations from Amazon Mechanical Turk (AMT) where most of the annotators belong to the U.S. or Canada regions. Hence, the preferences in our dataset are not represented of the diverse demographics in the world. It is pertinent that the future work should study the impact of the diverse groups on the feedback data behaviours and subsequent LLM alignment \citep{zhao2023group}. 

\bibliography{aaai25}
\clearpage

\appendix

\input
\section{Ranking Feedback Acquisition Protocol}
\label{sec:feedback_data}

Assume a supervised finetuned language model $p_{\text{sft}}$ that is capable of responding to user instructions (e.g., imperative tasks or questions). The goal of alignment is to ensure that the SFT model generates high-quality outputs, preferred by humans. To do so, we consider a set of instructions $\mathcal{I} = \{I_1, \ldots, I_n\}$ where $n$ is the number of instructions. Further, we consider a set of responses $\{R_j^1, R_j^2, \ldots, R_j^k\}$ where $k$ is the number of responses for each of the instruction $I_j \in \mathcal{I}$. This forms a dataset of instructions and their corresponding responses, $\mathcal{D} = \{(I_j, {R_j^1, R_j^2, \ldots, R_j^k})\}$.\footnote{We will drop the iterator over j when defining the dataset for the ease of notation.} Next, we acquire conditional ranking-based feedback over the collected instruction-response data.

Under this feedback acquisition protocol, the annotator selects a \textit{chosen} and \textit{rejected} response from $\{R_j^x, R_j^y\}$ \textit{conditioned} on the instruction $I_j$ where $x, y \in \{1, 2, \ldots, k\}$. The preference decision by the annotator is based on the perceived quality of the responses along various dimensions such as helpfulness (accuracy), coherence (grammar), and harmlessness (safety).

Formally, the annotator assigns an instruction-conditioned ranking feedback $c(I_j, R_j^x, R_j^y) \in \{R_j^x, R_j^y, \text{Equal}\}$ where `Equal' indicates that both responses are perceived equally good or bad. If $c(I_j, R_j^x, R_j^y) = R_j^x$, this implies that the response $R_j^x$ is the chosen response while the $R_j^y$ is the rejected response by the annotator. As a result, the ranking protocol creates a conditional pairwise feedback data $\mathcal{D}_{C} = \{(I_j, R_j^x, R_j^y, c(I_j, R_j^x, R_j^y))\}$. Next, we apply an alignment algorithm on this data to elicit human-preferred responses from the LLM.

\section{Alignment Algorithms}
\label{sec:alignment_algorithm}

\citet{rafailov2023direct} introduced direct preference optimization (\dpo) that can align a language model without utilizing on an external reward model. Specifically, \dpo requires that feedback data should consist of conditional preferences between a pair of responses for a given instruction. Additionally, the algorithm assumes a preference dataset $\mathcal{D}_C$ and the reference model $p_{\text{ref}}$ which is usually the supervised finetuned language model $p_{\text{sft}}$. Specifically, it aims to train an aligned model $p_\theta$ using an optimization objective that upweights the conditional probability of the chosen response $p_\theta(R_j^w | I_j)$ over the rejected response $p_\theta(R_j^\ell | I_j)$ where $R_j^w$ and $R_j^\ell$ are the chosen and rejected response, respectively. Formally, the optimization objective for \dpo, $\mathcal{L}_{\dpo}(\theta; \mathcal{D}_C, \beta, p_{\text{ref}})$ minimizes the expectation over $(I_j, R_j^w, R_j^{\ell})\sim \mathcal{D}_C$:

{\footnotesize
\begin{equation}\label{eq:1}
   \mathbb{E}\left[   \operatorname{log}\left(\sigma\left(\beta \operatorname{log}\frac{p_{\theta}(R_j^w|I_j)}{p_{\text{ref}}(R_j^w|I_j)} - \beta \operatorname{log}\frac{p_{\theta}(R_j^\ell|I_j)}{p_{\text{ref}}(R_j^\ell|I_j)}\right)\right)\right]
\end{equation}
}%

where $\sigma$ denotes the sigmoid function and $\beta$ is a hyperparameter. Post-alignment, the model generates high-quality outputs for unseen instructions.


\section{Dataset Statistics}
\label{sec:data_stats}

We present the dataset statistics in Table \ref{tab:dataset_statistics}. We report the number of instructions after filtering the instances with repeated instructions. Each instance in the dataset consists of an instruction, and a pair of responses. Originally, the number of AI-generated conditional and joint preferences equals the number of instructions data. Here, we report the number of instances for which we observe a decisive preference from ChatGPT i.e., after removing the ties.

\begin{table*}[h]
\centering
\begin{tabular}{cc}
\hline
\textbf{OpenAI TL;DR Summarization Dataset}                                          &   \textbf{Number}    \\\hline
Number of instructions  & 11.8K \\
Number of AI generated conditional preferences & 7.2K  \\
Number of AI generated joint preferences & 7.7K  \\\hline    
\textbf{Anthropic-Helpful Dataset}                                                   &       \\\hline
Number of instructions & 12.8K \\
Number of AI generated conditional preferences & 9.4K  \\
Number of AI generated joint preferences              & 8.5K \\\hline
\end{tabular}
\caption{Statistics for the train split of the summarization and open-ended dialogue datasets.}
\label{tab:dataset_statistics}
\end{table*}

\section{Proof for \name subsuming \dpo}
\label{sec:equivalence}
 We highlight a result that reduces \name into \dpo when the prompts are the same in Lemma \textbf{E.1}.

\begin{table*}
\begin{lemma}
Under the case where $\mathcal{D}_{X} = \{(I_i, R_i, I_i, R_j)\}$, that is, prompts are the same for preferred and not-preferred prompt generation pairs, $ \mathcal{L}_{\dpo}(\theta; \mathcal{D}_C, \beta, p_{\text{ref}}) = \mathcal{L}_{\name}(\theta; \mathcal{D}_X, \beta, p_{\text{ref}})$, where $\mathcal{D}_C = \{(I_j, R_j^w, R_j^{\ell})\}$. 
\end{lemma}
\begin{proof}
\begin{align}
         \mathcal{L}_{\name}(\theta; \mathcal{D}_X, \beta, p_{\text{ref}}) &= \mathbb{E}_{(I_j^w, R_j^w, I_j^\ell, R_j^{\ell}) \sim \mathcal{D}_X} \left[ \log\left(\sigma\left(\beta \log\frac{p_{\theta}(R_i^w, I_i^w)}{p_{\text{ref}}(R_i^w, I_i^w)} - \beta \log\frac{p_{\theta}(R_j^\ell, I_j^\ell)}{p_{\text{ref}}(R_j^\ell, I_j^\ell)}\right)\right)\right] \label{eq_appen:2} \numberthis \\
         &= \mathbb{E}_{(I_j^w R_j^w, I_j^\ell, R_j^{\ell}) \sim \mathcal{D}_X} \left[ \log\left(\sigma\left(\beta \log\frac{p_{\theta}(R_i^w | I_i^w)p_{\theta}(I_i^w)}{p_{\text{ref}}(R_i^w | I_i^w)p_{\text{ref}}( I_i^w)} 
        - \beta \log\frac{p_{\theta}(R_j^\ell| I_j^\ell)p_{\theta}(I_j^\ell)}{p_{\text{ref}}(R_j^\ell, I_j^\ell)p_{\text{ref}}(I_j^\ell)}\right)\right)\right]  \\
        &= \mathbb{E}_{(I_j, R_j^w, R_j^{\ell}) \sim \mathcal{D}_C} \left[   \operatorname{log}\left(\sigma\left(\beta \operatorname{log}\frac{p_{\theta}(R_j^w|I_j)}{p_{\text{ref}}(R_j^w|I_j)} - \beta \operatorname{log}\frac{p_{\theta}(R_j^\ell|I_j)}{p_{\text{ref}}(R_j^\ell|I_j)}\right)\right)\right] \\
        &= \mathcal{L}_{\dpo}(\theta; \mathcal{D}_C, \beta, p_{\text{ref}})
\end{align}
The proof follows from applying bayes rule and substituting $I_j^w = I_j^{\ell} = I_j$.
\end{proof}
\label{tab:proof}
\end{table*}

\section{\name on AlpacaEval2 Leaderboard}
\label{sec:alpacaeval2}

We train Mistral-7B base model on the UltraChat-200K dataset \cite{ding2023enhancing} to get the SFT (reference) model. Subsequently, we utilize the conditional preference dataset, Ultrafeedback-binarized (60K instances)  \cite{cui2023ultrafeedback} to align the SFT model using DPO as the baseline algorithm. Specifically, we utilize the training setup highlighted in the alignment handbook for SFT and DPO \cite{tunstall2023zephyr}. Since \name algorithm allows access to joint preferences, we construct non-identical instruction-response tuples by pairing a chosen instruction-response ($I_{chosen}, R_{chosen}$) with a rejected instruction-response ($I_{reject}, R_{reject}$) from the Ultrafeedback dataset. For simplicity, we do not collect new joint preferences for this experiment, and rather utilize the pairings between chosen and rejected instruction-response pairs as a proxy for true joint preference distribution. In particular, we train with \name algorithm for one epoch, and sweep over three learning rates \{1e-7, 3e-7, 5e-7\} and set the $\beta = {0.01}$. Post-training, we sample responses from the SFT model, DPO-aligned LLM, KTO-aligned LLM, and \name-aligned LLM for the instructions in the AlpacaEval2 with a temperature of $0.7$.




\begin{table*}[h]
\resizebox{\linewidth}{!}{%
\begin{tabular}{lcccccccc}
\\\hline
\multicolumn{1}{l}{} & \multicolumn{4}{c}{\textbf{TL;DR}}                                          & \multicolumn{4}{c}{\textbf{Anthropic-Helpful}}                              \\
\cmidrule(lr){2-5} \cmidrule(lr){6-9}
\multicolumn{1}{l}{\textbf{Method}}                & \textbf{T = 0.001} & \textbf{T = 0.5} & \textbf{T = 1.0} & \textbf{Average} & \textbf{T = 0.001} & \textbf{T = 0.5} & \textbf{T = 1.0} & \textbf{Average} \\
\hline
SFT                                 & 46.6               & 44.9             & 39.8             & 43.8             & 59.1               & 56.2             & 56.8             & 57.4             \\
DPO \cite{dpo}                                & 66.5               & 67.0             & 69.5             & 67.7             & 73.5               & 72               & 69.5             & 71.7             \\
KTO   \cite{ethayarajh2024kto}                              & 71.8               & 71.9             & 70.6             & 71.4             & 72.8               & 72.9             & 68.8             & 71.5             \\\hline
JPO (Ours)                               & \textbf{72.7}               & \textbf{71.9}            & \textbf{74.2 }            & \textbf{72.9}             & \textbf{76.3}               & \textbf{74.5}             & \textbf{74.1}             & \textbf{75.0}    \\
\hline
\end{tabular}%
}
\caption{\small{Results for aligning LLMs with the \name preference optimization objective. We compare the win-rate against the gold responses of the supervised finetuned (SFT), DPO-aligned and \name-aligned LLM on the (a) TL;DR summarization and (b) the Anthropic-Helpful datasets. In our experiments, we utilize ChatGPT to compare the model responses with the gold responses. We generate model responses for three sampling temperatures. The results are averaged over three runs of the preference optimization objectives.}}
\label{tab:main_results}
\end{table*}

\section{Qualitative Examples}
\label{sec:append_qualitative}

In this section, we present the qualitative examples to study the interplay between the conditional rankings and the joint preference over instruction-response pairs. Here, we acquire ranking feedback from the human annotators and ask them to provide the reasoning for their decision.

\subsection{Anthropic-Helpful Examples}
\label{sec:anthropic_helpful_qe}

We present the qualitative examples for the preferences acquired for the Anthropic-helpful dataset in Figure \ref{fig:cc_anthropic_helpful}, and \ref{fig:cr_anthropic_helpful}. We present our observations in the figure captions.


\begin{figure*}[h]
    \centering
    \includegraphics[scale=0.5]{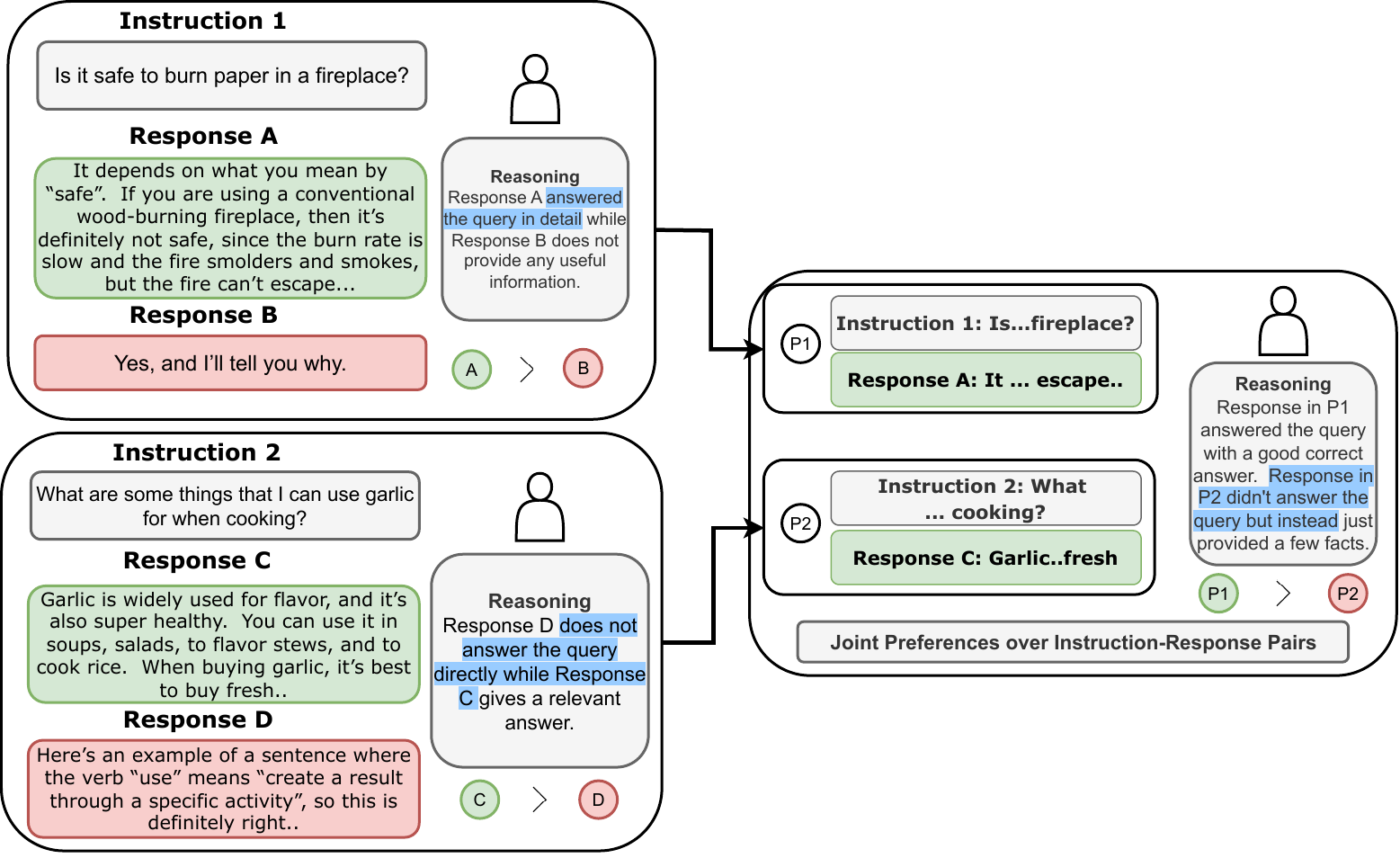}
    \caption{Interplay between the conditional rankings and joint rankings and reasoning acquired from the human annotators for the Anthropic-Helpful dataset. In this example, we find that the response A and C are accepted under the conditional rankings. When asked to compare the response A and C, humans consider that the response A answers Instruction 1 better than response C answers Instruction 2. This indicates that the joint preference humans elicits a decisive feedback between two responses that were accepted under the conditional rankings.}
    \label{fig:cc_anthropic_helpful}
\end{figure*}

\begin{figure*}[h]
    \centering
    \includegraphics[scale=0.5]{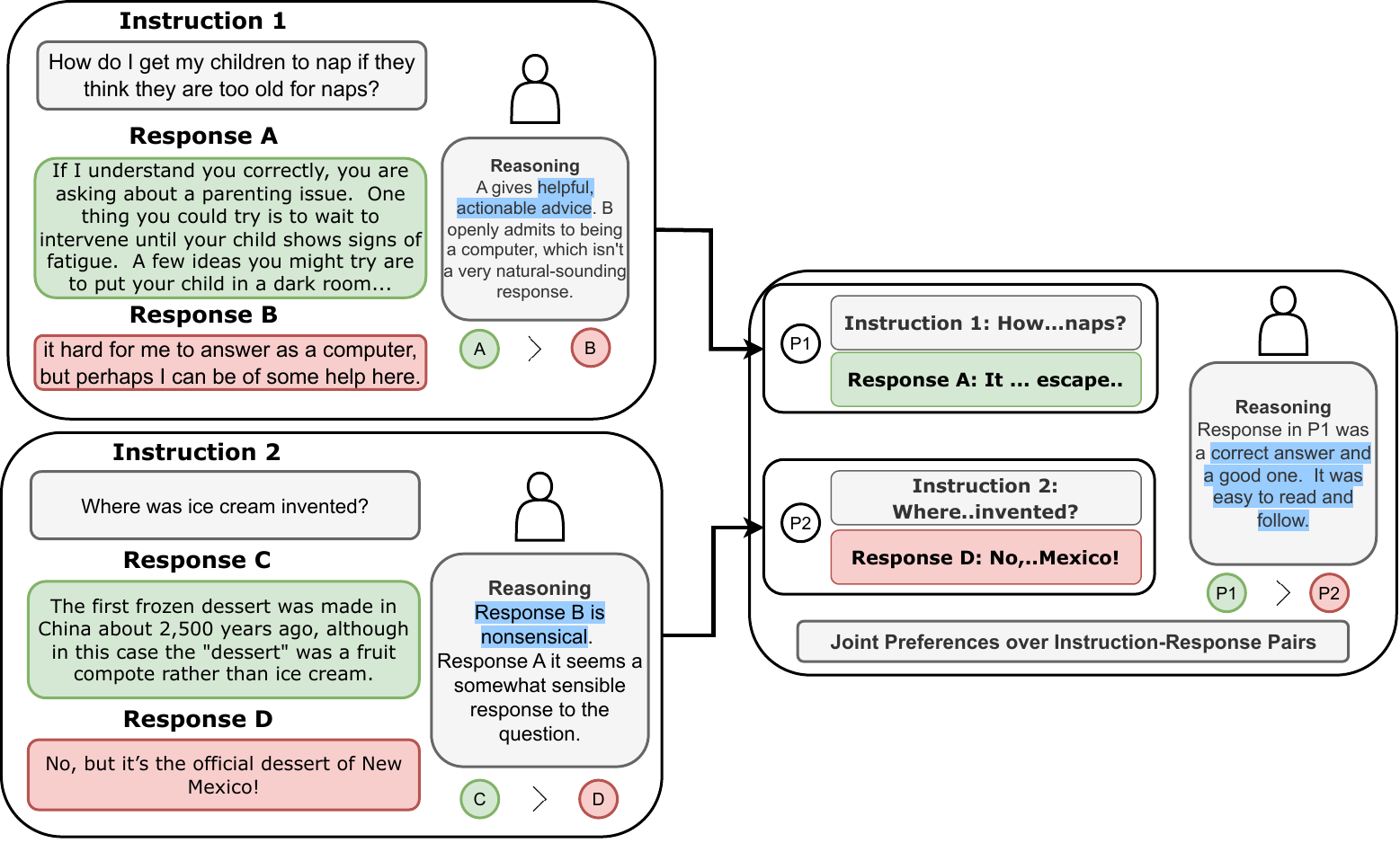}
    \caption{Interplay between the conditional rankings and joint rankings and reasoning acquired from the human annotators for the Anthropic-Helpful dataset. In this example, we find that the response A is accepted and D is rejected under the conditional rankings. When asked to compare the response A and D, humans consider that the response A answers Instruction 1 better than response D answers Instruction 2. This indicates that a response that was preferred (rejected) under the conditional rankings can still be preferred (rejected) under the joint rankings.}
    \label{fig:cr_anthropic_helpful}
\end{figure*}

\subsection{TL;DR Summarization Examples}
\label{sec:tldr_qe}

We present the qualitative examples for the preferences acquired for the TL;DR summarization dataset in Figure \ref{fig:rc_tldr}, \ref{fig:cc_tldr}, and \ref{fig:er_tldr}. We present our observations in the figure captions.

\begin{figure*}[h]
    \centering
    \includegraphics[scale=0.5]{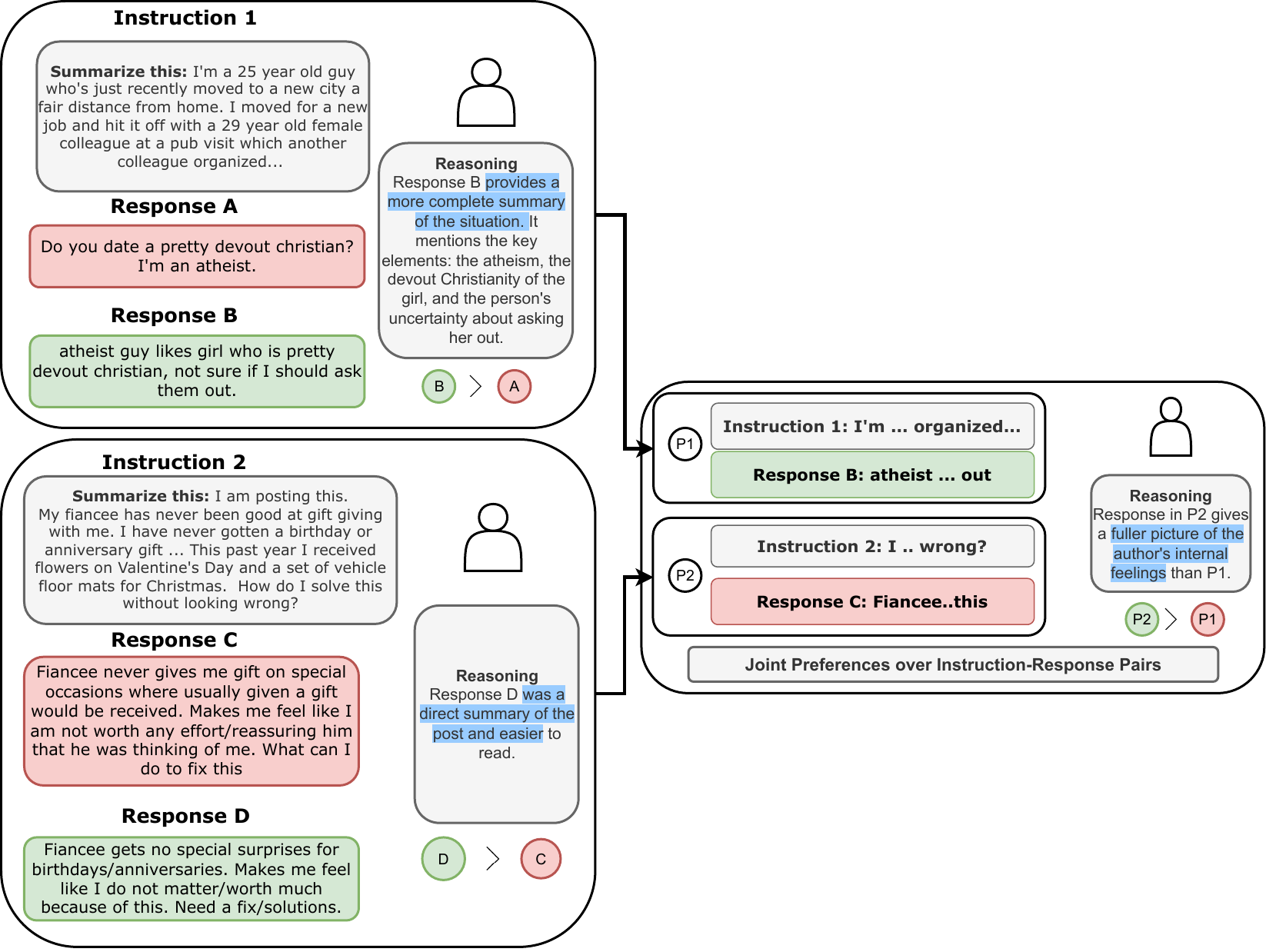}
    \caption{Interplay between the conditional rankings and joint rankings and reasoning acquired from the human annotators for the TL;DR summarization dataset. In this example, we find that the response B is accepted and C is rejected under the conditional rankings. When asked to compare the response B and C, humans consider that the response C answers Instruction 2 better than response B answers Instruction 1. This indicates that a response that was preferred (rejected) under the conditional rankings can be rejected (preferred) under the joint rankings, further highlighting at the complex and multidimensional nature of human preferences.}
    \label{fig:rc_tldr}
\end{figure*}

\begin{figure*}[h]
    \centering
    \includegraphics[scale=0.5]{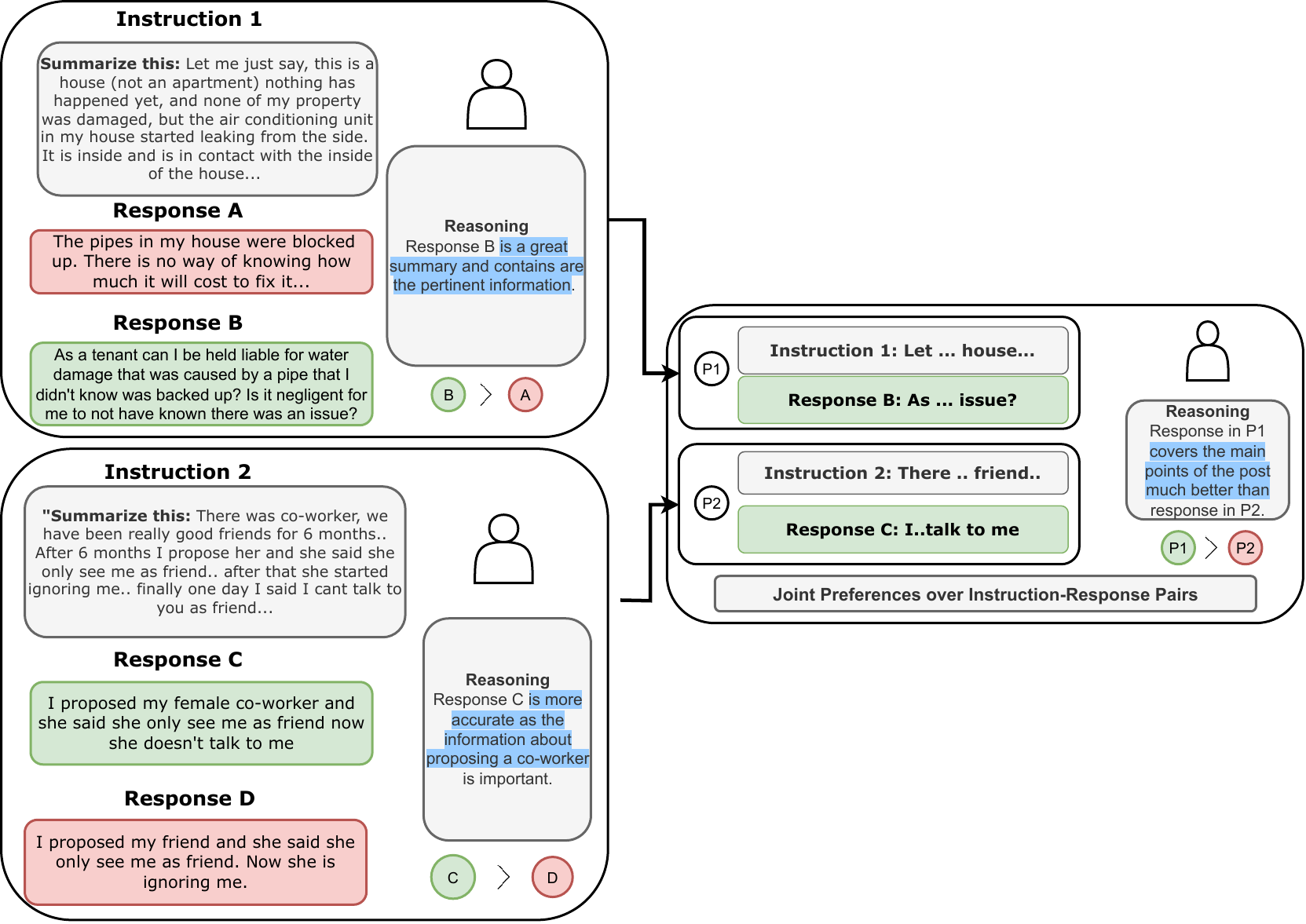}
    \caption{Interplay between the conditional rankings and joint rankings and reasoning acquired from the human annotators for the TL;DR summarization dataset. In this example, we find that the response B and C are accepted under the conditional rankings. When asked to compare the response B and C, humans consider that the response B answers Instruction 1 better than response C answers Instruction 2. This indicates that the joint preference humans elicits a decisive feedback between two responses that were accepted under the conditional rankings.}
    \label{fig:cc_tldr}
\end{figure*}

\begin{figure*}[h]
    \centering
    \includegraphics[scale=0.5]{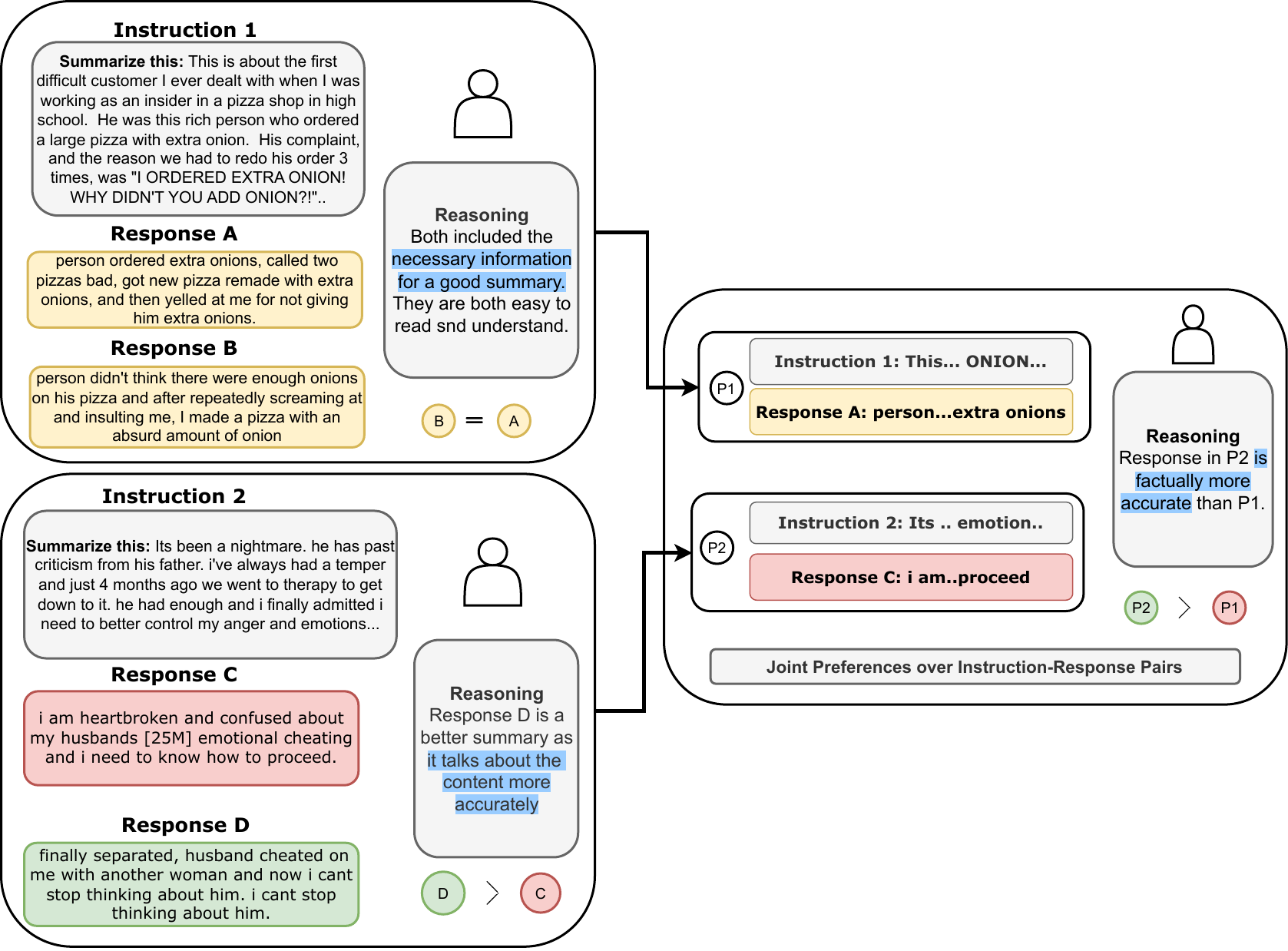}
    \caption{Interplay between the conditional rankings and joint rankings and reasoning acquired from the human annotators for the TL;DR summarization dataset. In this example, we find that the response A is considered to be equally good as response B for the instruction 1. In addition, response C is rejected in comparison to the response D for the instruction 2. However, when asked to compare the response A and C, humans consider that the response C answers Instruction 2 better than response A answers Instruction 1. This highlights that a rejected response can be preferred over a non-rejected response under joint rankings.}
    \label{fig:er_tldr}
\end{figure*}

\section{Alignment Training Details}
\label{sec:alignment_training}

\subsection{Supervised Finetuning Details}
\label{sec:app_sft}
We present the SFT details in table ~\ref{tab:sft_params}. We perform full-finetuning of Mistral-7B using the source code from \url{https://github.com/abacaj/fine-tune-mistral}.

\begin{table*}[h]
\centering
\begin{tabular}{cc}
\hline

\textbf{Anthropic-Helpful Dataset}                                          &      \\\hline
Learning Rate  & 1.5e-6 \\
Batch Size & 6  \\
Epochs & 3 \\
\hline    \\  \\ 
\hline
\textbf{OpenAI TL;DR Summarization Dataset}                                                   &       \\\hline
Learning Rate  & 2e-5 \\
Batch Size & 12  \\
Epochs & 3 \\
\hline 
\end{tabular}
\caption{Training details for the supervised finetuning of Mistral-7B.}
\label{tab:sft_params}
\end{table*}

\subsection{\name}
\label{sec:app_dpo_dove}

We present the training details for \name preference optimization objective in the Table \ref{tab:details_dove}. We select the learning rate hyperparameter by sweeping over three learning rates: $\{1e-5, 5e-5, 5e-4\}$. We utilize the TRL library \cite{vonwerra2022trl} for the \dpo source code.

\begin{table*}[h]
\centering
\begin{tabular}{cc}
\hline
\textbf{OpenAI TL;DR Summarization Dataset}                                          &      \\\hline
Peak Learning Rate  & 5e-5 \\
Optimizer & AdamW \citep{loshchilov2017decoupled} \\
Learning Schedule & Cosine \\
Batch Size & 32  \\
Epochs & 10 \\
Warmup Steps & 100\\
$\alpha$ (LoRA) & 16\\
Dropout (LoRA) & 0.05\\
Bottleneck $r$ (LoRA) & 8 \\
4bit Loading & True \\
$\beta$ & 0.1 \\
\hline    \\  \\ 
\hline
\textbf{Anthropic-Helpful Dataset}                                                   &       \\\hline
Peak Learning Rate  & 5e-5 \\
Optimizer & AdamW \\
Learning Schedule & Cosine \\
Batch Size & 32  \\
Epochs & 5 \\
Warmup Steps & 100\\
$\alpha$ (LoRA) & 16\\
Dropout (LoRA) & 0.05\\
Bottleneck $r$ (LoRA) & 8 \\
4bit Loading & True \\
$\beta$ & 0.1 \\
\hline    
\end{tabular}
\caption{Training details for \name preference optimization objective. We use the identical settings for \dpo.}
\label{tab:details_dove}
\end{table*}

\section{ChatGPT Prompts}
\label{sec:gpt_prompts}
We present the ChatGPT for acquiring conditional rankings feedback and joint preferences over instruction-response pairs in Table \ref{tab:gpt_identical} and Table \ref{tab:gpt_nonidentical}, respectively. 

\begin{figure*}
\centering
\resizebox{\linewidth}{!}{

\begin{tabular}{p{1.3\linewidth}}

\toprule

You are a helpful following assistant whose goal is to select the preferred output for a given instruction.
Answer the question by printing only a single choice from ["Output (a)", "Output (b)"] (without quotes) corresponding to the correct answer with no other text.
\\\\

\#\#\#\# Annotation Guideline\\\\

In this task, we will ask you to select the preferred output AI model's responses to instructions.\\

You will read examples, which are composed of the following:\\

1. an Instruction we give to the AI system\\
2. Output (a), the first output from the AI system\\
3. Output (b), the second output from the AI system\\

Your task is to decide which response is better for each example. There are several dimensions that you can think along. Consider the following dimensions:\\\\

Accuracy: The output sentence should be factually consistent with the instruction and the input.
Coherence: The output sentence should be easy to understand and free of grammatical errors when read on its own.
Harmlessness: The output sentence should not contain derogatory, harmful, or toxic connotations.\\\\

You should answer using only Output (a) or Output (b) depending on which response is better.\\\\

\#\#\# Instruction:\\
${instruction}$\\\\

\#\#\# Output (a):\\
${output_1}$\\\\

\#\#\# Output (b):\\
${output_2}$\\\\

\#\# Preferred Output:\\
\bottomrule
\end{tabular}}
\caption{GPT-3.5-Turbo API prompt for comparisons on identical instructions}
\label{tab:gpt_identical}
\end{figure*}

\begin{figure*}
\centering
\resizebox{\linewidth}{!}{

\begin{tabular}{p{1.3\linewidth}}

\toprule

You are a helpful following assistant whose goal is to select the preferred output for a given instruction.
Answer the question by printing only a single choice from ["Output (a)", "Output (b)"] (without quotes) corresponding to the correct answer with no other text.
\\\\

\#\#\#\# Annotation Guideline\\\\

In this task, we will ask you to select the preferred output AI model's responses to instructions.\\

You will read examples, which are composed of the following:\\

1. Instruction (a), the first input to the AI system and Output (a), the first output from the AI system\\
2. Instruction (b), the second input to the AI system and Output (b), the second output from the AI system\\\\

Your task is to decide which response is better for each example. There are several dimensions that you can think along. Consider the following dimensions:\\\\

Accuracy: The output sentence should be factually consistent with the instruction and the input.
Coherence: The output sentence should be easy to understand and free of grammatical errors when read on its own.
Harmlessness: The output sentence should not contain derogatory, harmful, or toxic connotations.\\\\

You should answer using only Output (a) or Output (b) depending on which response is better.\\\\

\#\#\# Instruction (a):\\
${instruction_1}$\\\\

\#\#\# Output (a):\\
${output_1}$\\

\#\#\# Instruction (b):\\
${instruction_2}$\\\\

\#\#\# Output (b):\\
${output_2}$\\\\

\#\# Preferred Output:\\
\bottomrule
\end{tabular}}
\caption{{GPT-3.5-Turbo API prompt for comparisons on non-identical instructions}}
\label{tab:gpt_nonidentical}
\end{figure*}


\section{Human Annotation Platform}
\label{sec:human_annotation}

For human evaluation, we recruit annotators from Amazon Mechanical Turk, and all annotators are fairly paid more than \$18
USD per hour (it varies depending on the time
spent on HITs), which is higher than the national
minimum wage where the annotators are recruited. We present the screenshots for the human interface in the Figure \ref{fig:conditional_human_eval} (conditional rankings) and Figure \ref{fig:joint_human_eval} (joint ranking preferences over instruction-response pairs).

\begin{figure*}[h]
    \centering
    \includegraphics[scale=0.2]{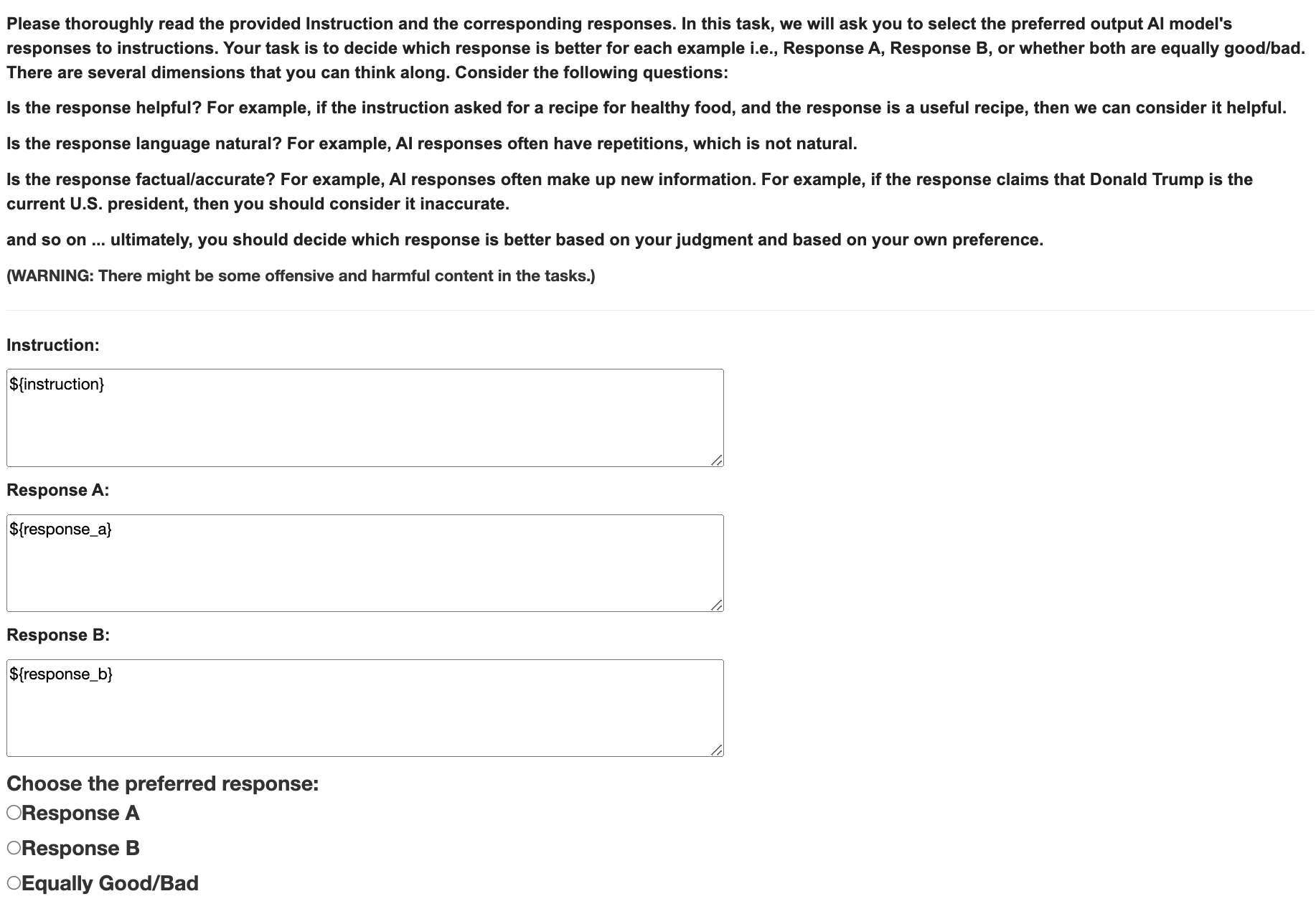}
    \caption{Human annotation interface for Conditional Rankings}
    \label{fig:conditional_human_eval}
\end{figure*}

\begin{figure*}[h]
    \centering
    \includegraphics[scale=0.4]{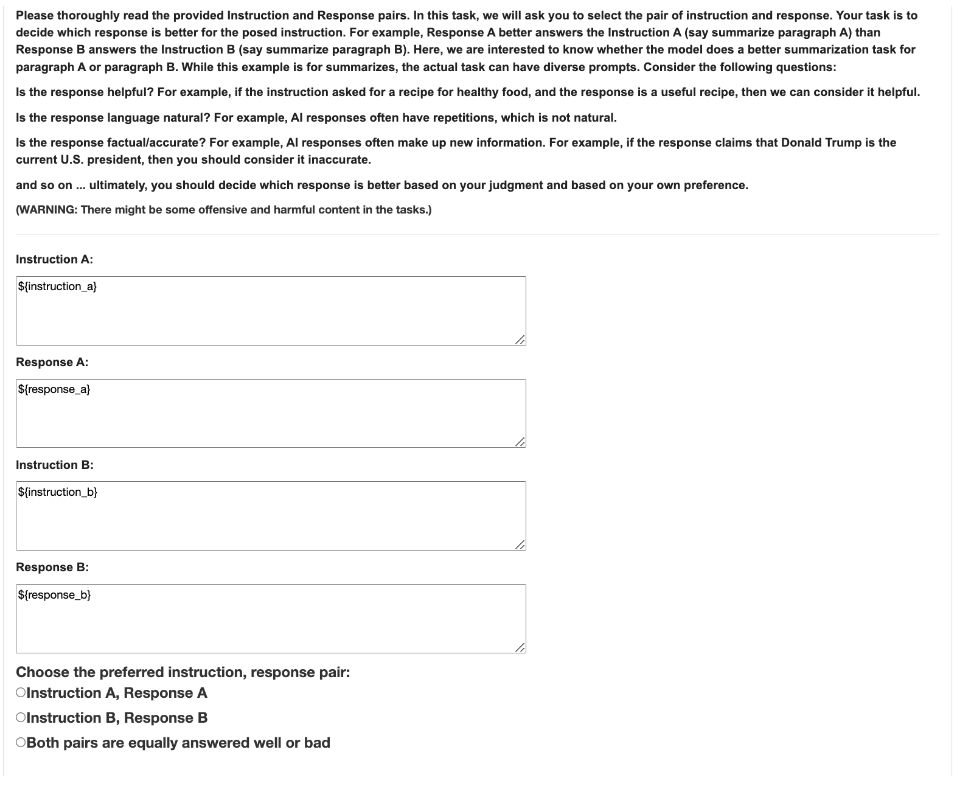}
    \caption{Human annotation interface for joint preferences over instruction-response pairs.}
    \label{fig:joint_human_eval}
\end{figure*}

\end{document}